\documentclass[12pt]{article}
\usepackage{epsfig,amsmath,amssymb,amsfonts,amstext,amsthm,mathrsfs}
\usepackage{latexsym,graphics,epsf,epsfig,psfrag,color}
\usepackage{cite}

\newcommand{\mmd}{\text{MMD}}

\newtheorem{definition}{\textbf{Definition}}
\newtheorem{corollary}{\textbf{Corollary}}
\newtheorem{lemma}{\textbf{Lemma}}
\newtheorem{theorem}{\textbf{Theorem}}
\newtheorem{proposition}{\textbf{Proposition}}

\newcommand{\nn}{\nonumber}

\newcommand{\mE}{\mathbb{E}}

\newcommand{\cH}{\mathcal{H}}

\newcommand{\cP}{\mathcal{P}}
\newcommand{\cI}{\mathcal{I}}

\DeclareMathAlphabet{\matheuf}{U}{euf}{m}{n}

\begin{document}

\vspace*{-2cm}

\begin{center}
  \baselineskip 1.3ex {\Large \bf Nonparametric Detection of Geometric Structures over Networks
\let\thefootnote\relax\footnote{The material in this paper was presented
in part at IEEE International Workshop on Machine Learning for Signal Processing (MLSP), Reims, France, September, 2014\cite{zou2014kernel}.
}
\let\thefootnote\relax\footnote{The work of S. Zou and Y. Liang was supported by a National Science Foundation
CAREER Award under Grant CCF-10-26565. The work of H. V. Poor was supported by National Science Foundation under Grants DMS-11-18605 and ECCS-13-43210.}
}
\\
 \vspace{0.15in} Shaofeng Zou, Yingbin Liang, H. Vincent Poor
\let\thefootnote\relax\footnote{Shaofeng Zou and Yingbin Liang are with the Department of Electrical
Engineering and Computer Science, Syracuse University, Syracuse, NY 13244 USA (email: \{szou02,yliang06\}@syr.edu).
H. Vincent Poor is with the Department of Electrical Engineering, Princeton University, Princeton, NJ 08544 USA (email: poor@princeton.edu).
}
\end{center}

\begin{abstract}
\baselineskip 3.5ex
Nonparametric detection of existence of an anomalous structure over a network is investigated. Nodes corresponding to the anomalous structure (if one exists) receive samples generated by a distribution $q$, which is different from a distribution $p$ generating  samples for other nodes. If an anomalous structure does not exist, all nodes receive samples generated by $p$. It is assumed that the distributions $p$ and $q$ are arbitrary and unknown. The goal is to design statistically consistent tests with probability of errors converging to zero as the network size becomes asymptotically large. Kernel-based tests are proposed based on maximum mean discrepancy that measures the distance between mean embeddings of distributions into a reproducing kernel Hilbert space. Detection of an anomalous interval over a line network is first studied. Sufficient conditions on minimum and maximum sizes of candidate anomalous intervals are characterized in order to guarantee the proposed test to be consistent. It is also shown that certain necessary conditions must hold to guarantee any test to be universally consistent. Comparison of sufficient and necessary conditions yields that the proposed test is order-level optimal and nearly optimal respectively in terms of minimum and maximum sizes of candidate anomalous intervals. Generalization of the results to other networks is further developed. Numerical results are provided to demonstrate the performance of the proposed tests.
\end{abstract}

%
%
\section{Introduction}\label{sec:introduction}

We are interested in a type of problems, the goal of which is to detect existence of an anomalous structure over a network. Each node in the network observes a random sample. An anomalous structure, if it exists, corresponds to a cluster of nodes in the network that take samples generated by a distribution $q$. All other nodes in the network take samples generated by a distribution $p$ that is different from $q$. If there does not exist an anomalous structure, then all nodes receive samples generated by $p$. The distributions $p$ and $q$ are \emph{arbitrary} and \emph{unknown a priori}. Designed tests are required to distinguish between the null hypothesis (i.e., no anomalous structure exists) and the alternative hypothesis (i.e., there exists an anomalous structure). Due to the fact that the anomalous structure may be one of a number of candidate structures in the network, this is a composite hypothesis testing problem. In this paper, we first study the problem of detecting existence of an anomalous interval over a line network, and then generalize our approach to higher dimensional networks.


Such a problem models a variety of applications. For example, in sensor networks, sensors are deployed over a large range of space. These sensors take measurements from the environment in order to determine whether or not there is intrusion of an anomalous object. Such intrusion typically activates only a few sensors that cover a certain geometric area. An alarm is then triggered if the network detects an occurrence of intrusion based on the sensors' measurements. Other applications can arise in detecting an anomalous segment of DNA sequences, detecting virus infection of computer networks, and detecting anomalous spot in images.

As an interesting topic, detecting existence of an anomalous geometric structure in networks has been intensively studied in the literature as we review in Subsection \ref{sec:relatedwork}. However, previous studies focused on \emph{parametric} or \emph{semiparametric} models, which assume that samples are generated by known distributions such as Gaussian or Bernoulli distributions, or the two distributions are known to have mean shift. Such parametric models may not always hold in real applications. In many cases, distributions can be arbitrary, and may not be Gaussian or Bernoulli. They may not differ in mean either. The distributions may not even be known in advance. Hence, it is desirable to develop nonparametric tests that are universally applicable to arbitrary distributions.

In contrast to previous studies, we study the \emph{nonparametric} problem of detecting an anomalous structure, in which distributions can be \emph{arbitrary} and \emph{unknown a priori}. In order to deal with nonparametric models, we apply mean embedding of distributions into a reproducing kernel Hilbert space (RKHS) \cite{Berl2004,Srip2010} (see \cite{Scho2002} for an introduction of RKHS). The idea is to map probability distributions into an RKHS associated with an appropriate kernel such that distinguishing between two probability distributions can be carried out by evaluating the distance between the corresponding mean embeddings in the RKHS. This is valid because the mapping is shown to be injective for various kernels \cite{Fuku2008} such as Gaussian and Laplacian kernels. The main advantage of such an approach is that the mean embedding of a distribution can be easily estimated based on samples. This approach has been applied to solving the two sample problem in \cite{Gretton2012}, in which the quantity of {\em maximum mean discrepancy (MMD)} was used as a metric of  distance between mean embeddings of two distributions. In this paper, we apply MMD as a metric to construct tests for the nonparametric detection problem of interest.


We are interested in the asymptotic scenario in which the network size goes to infinity and the number of candidate anomalous structures scales with the network size. Thus, the number of sub-hypotheses under the alternative hypothesis also increases, which causes the composite hypothesis testing problem to be difficult. On the other hand, since the distributions can be arbitrary, it is in general difficult to exploit properties of the distributions such as mean shift to detect existence of an anomalous structure. Furthermore, as the network size becomes large, in contrast to parametric models in which the mean shift can scale with the network size, here it is necessary that the numbers of samples within and outside of each anomalous structure should scale with the network size fast enough in order to provide more accurate information about both distributions $p$ and $q$ and guarantee asymptotically small probability of error. Thus, the problem amounts to characterize how the minimum and maximum sizes of all candidate anomalous structures should scale with the network size in order to consistently detect the existence of an anomalous structure.

In this paper, we adopt the following notations to express asymptotic scaling of quantities with the network size $n$:
\begin{list}{$\bullet$}{\topsep=0.ex \leftmargin=0.3in \rightmargin=0.in \itemsep =-0.05in}
\item $f(n)=O(g(n))$: there exist $k,n_0>0$ s.t.\ for all $ n>n_0$, $|f(n)|\leq k|g(n)|$;
\item $f(n)=\Omega(g(n))$: there exist $k,n_0>0$ s.t.\ for all $ n>n_0$, $f(n)\geq kg(n)$;
\item $f(n)=\Theta(g(n))$: there exist $k_1,k_2,n_0>0$ s.t.\ for all $ n>n_0$, $k_1g(n)\leq f(n)\leq k_2g(n)$;
\item $f(n)=o(g(n))$: for all $k>0$, there exists $n_0>0$ s.t.\ for all $ n>n_0$, $|f(n)|\leq kg(n)$;
\item $f(n)=\omega(g(n))$: for all $k>0$, there exists $n_0>0$ s.t.\ for all $ n>n_0$, $|f(n)|\geq k|g(n)|$.
\end{list}

\subsection{Main Contributions}

We adopt the MMD to construct nonparametric tests for various networks, which is based on kernel embeddings of distributions into an RKHS. Our main contribution lies in comprehensive analysis of the performance guarantee for the proposed tests. For the problem of detecting an anomalous interval over a line network, we show that as the network size $n$ goes to infinity, if the minimum size $I_{\min}$ of candidate anomalous intervals satisfies $I_{\min}=\Omega(\log n)$, and the maximum size  $I_{\max}$ of candidate anomalous intervals satisfies $n-I_{\max}=\Omega(\underset{\text{arbitrary k number of }\log}{\underbrace{\log\cdots\cdots\log }n})$, then the proposed test is consistent, i.e., the probability of error is asymptotically small. We further derive necessary conditions on $I_{\min}$ and $I_{\max}$ that any test must satisfy in order to be universally consistent for arbitrary $p$ and $q$. Comparison of sufficient and necessary conditions yields that the MMD-based test is order-level optimal in terms of $I_{\min}$ and nearly order-level optimal in terms of $I_{\max}$. We further generalize such analysis to other networks and obtain similar type of results. Our results also demonstrate the impact of geometric structures on performance guarantee of tests.

Our technical analysis is very different from that for parametric problems. The obvious difference is due to significantly different approaches applied to the two types of problems. Furthermore, the nonparametric nature also affects the asymptotic formulation of the problem. The lower and upper bounds (such as $I_{\min}$ and $I_{\max}$ in line networks) on the sizes of all candidate anomalous structures must scale with the network size in order to guarantee enough samples in and outside the anomalous structure if it occurs. This is significantly different from parametric models where problems can still be well posed even with a single node or the entire network being anomalous, so long as a certain distribution parameter (such as mean shift between the two distributions) scales with the network size. Consequently, the asymptotic analysis for the nonparametric problem requires considerable new technical developments.

Although the kernel-based approach has been used to solve various machine learning problems, it is not widely applied to solving detection problems with only few exceptions such as the two sample problem \cite{Gretton2012}. Since the nature of our problem necessarily involves geometric structures in networks, the technical analysis requires substantial efforts to deal with scaling of the size of geometric structures and analyze the impact of geometry on consistency of tests, which are not captured in the two sample problem.

\subsection{Related Work}\label{sec:relatedwork}

%
%
%

Detecting existence of an anomalous geometric structure in large networks has been extensively studied in the literature. A number of studies focused on networks with nodes embedded in a lattice such as one dimensional line and square. In \cite{Arias2005}, the network is assumed to be embedded in a $d$-dimensional cube, and geometric structures such as line segments, disks, rectangles and ellipsoids associated with nonzero-mean Gaussian random variables need to be detected out of other nodes associated with zero-mean Gaussian noise variables. A multi-scale approach was proposed and its optimality was analyzed. In \cite{Walt2010}, detection of spatial clusters under the Bernoulli model over a two-dimensional space was studied, and a new calibration of the scan statistic was proposed, which results in optimal inference for spatial clusters. In \cite{Paci2004}, the problem of identifying a cluster of nodes with nonzero-mean values from zero-mean noise variables over a random field was studied.

Further generalization of the problem has also been studied, when network nodes are associated with a graph structure, and existence of an anomalous cluster or an anomalous subgraph of nodes needs to be detected. In \cite{Arias2008}, an unknown path corresponding to nonzero-mean variables needs to be detected out of zero-mean variables in a network with nodes connected in a graph. In \cite{Adda2010}, for various combinatorial and geometric structures of anomalous objects, conditions were established under which testing is possible or hopeless with a small risk. In \cite{Arias2011}, the cluster of anomalous nodes can either take certain geometric shapes or be connected as subgraphs. Such structures associated with nonzero-mean Gaussian variables need to be detected out of zero-mean variables. In \cite{Sharp2013a} and \cite{Sharp2013b}, network properties of anomalous structures such as small cut size were incorporated in order to assist successful detection. More recently, in \cite{qian2014}, the problem of detecting connected sub-graph with elevated mean out of zero-mean Gaussian random variables was studied. An algorithm  was proposed to characterize the family of all connected sub-graphs in terms of linear matrix inequalities. The minimax optimality of such an approach was further established in \cite{qian2014efficient} for exponential family on 1-D and 2-D lattices.

Our problem differs from all of the above studies due to its nonparametric nature, i.e., the distributions are assumed to be unknown and arbitrary.

\subsection{Organization of the Paper}

The rest of the paper is organized as follows. In Sections \ref{sec:model}, we describe the problem formulation in the context of a line network and introduce the MMD-based approach. In Section \ref{sec:line}, we present our results for line networks. In section \ref{sec:generalization}, we generalize our approach to other networks. In Section \ref{sec:num}, we provide numerical results, and finally in Section \ref{sec:con}, we conclude our paper with remarks on future work.

\section{Problem Statement and Preliminaries}\label{sec:model}


\subsection{Problem Statement}

In this subsection, we introduce our problem formulation in the context of line networks that we study in Section \ref{sec:line}. We describe generalization of this problem to other networks in Section \ref{sec:generalization} when we present the corresponding results for these networks.

\begin{figure}
\begin{center}
\includegraphics[width=6.5cm]{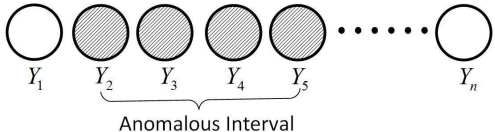}
\caption{A line network with an anomalous interval.}\label{fig:model}
\end{center}
\end{figure}

We consider a line network, which consists of nodes $1,\ldots,n$, as shown in Figure~\ref{fig:model}. We use $I$ to denote a subset of consecutive indices of nodes, which is referred to as an {\em interval}. Here, the length of an interval $I$ refers to the cardinality of $I$, and is denoted by $|I|$. We assume that any interval with the length between $I_{\min}$ and $I_{\max}$ can be a candidate anomalous interval, and collect all candidate anomalous intervals into the following set $\cI_n^{(a)}$
\begin{equation}\label{eq:Ia}
 \mathcal I_n^{(a)}=\{I:I_{\min}\leq |I|\leq I_{\max}\}.
\end{equation}
As we explain towards the end of this subsection, the two problem parameters $I_{\min}$ and $I_{\max}$ play an important role in determining whether the problem is well posed.


We assume that each node, say node $i$, is associated with a random variable, denoted by $Y_i$, for $i=1,\ldots,n$. We consider two hypotheses about the distributions of the line network. For the {\em null hypothesis} $H_0$, $Y_i$ for $i=1,\ldots,n$ are identical and independently distributed (i.i.d.) random variables, and are generated from a distribution $p$. For the {\em alternative hypothesis} $H_1$, there exists an interval $I \in \cI_n^{(a)} $ over which $Y_i$ (with $i \in I$) are i.i.d.\ and are generated from a distribution $q\neq p$, and otherwise, $Y_i$ are i.i.d.\ and generated from the distribution $p$. Thus, the alternative hypothesis is composite due to the fact that $\mathcal I_n^{(a)}$ contains multiple candidate anomalous intervals, and these intervals differentiate from each other by their length and location in the network. We further assume that under both hypotheses, each node generates only one sample. Putting the problem into a context, $H_0$ models the scenario when the observations $Y_i$ are background noise, and $H_1$ models the scenario when some $Y_i$ (for $i \in I$) are observations activated by an anomalous intrusion.

In contrast to previous work, we assume that the distributions $p$ and $q$ are {\em arbitrary} and {\em unknown a priori}.
For this problem, we are interested in the asymptotic scenario, in which the number of nodes goes to infinity, i.e., $n \rightarrow \infty$. The performance of a test for such a system is captured by two types of errors. The {\em type I error} refers to the event that samples are generated from the null hypothesis, but the detector determines an anomalous event occurs. We denote the probability of such an event as $P(H_1|H_0)$, or $P_{H_0}(error)$. The {\em type II error} refers to the case that an anomalous event occurs but the detector claims that samples are generated from the null hypothesis. We denote the probability of such an event as $P(H_0|H_1)$, or $P_{H_1}(error)$.
We define the following minimax risk to measure the performance of a test:
\begin{flalign}\label{eq:mmrisk}
  R_m^{(n)}=P(H_1|H_0)+\underset{I\in\mathcal I_n^{(a)}}{\max}P(H_0|H_{1,I}).
\end{flalign}
\begin{definition}\label{def:consistent}
  A test is said to be consistent if the minimax risk $R_m^{(n)}\rightarrow 0$, as $n\rightarrow \infty$.
\end{definition}


It can be seen from the definition of $\mathcal I_n^{(a)}$ that $I_{\min}$ and $I_{\max}$ determine the number of candidate anomalous intervals. Furthermore, if there exists an anomalous intervals, $I_{\min}$ determines the least number of samples generated by $q$ and $n-I_{\max}$ determines the least number of samples generated by $p$. As $n \rightarrow \infty$, to guarantee asymptotically small probability of error, both $I_{\min}$ and $I_{\max}$ must scale with $n$ to provide sufficient information about $p$ and $q$ in order to yield accurate distinction between the two hypotheses. This suggests that as the network becomes larger, only a large enough but not too large anomalous object can be detected. Therefore, our goal in this problem is to characterize how $I_{\min}$ and $I_{\max}$ should scale with the network size in order for a test to successfully distinguish between the two hypotheses. Such conditions on $I_{\min}$ and $I_{\max}$ can thus be interpreted as the resolution of the corresponding test.

\subsection{Preliminaries of MMD}
We provide brief introduction about the idea of mean embedding of distributions into RKHS \cite{Berl2004,Srip2010} and the metric of MMD. Suppose $\cP$ includes a class of probability distributions, and suppose $\cH$ is the RKHS associated with a kernel $k(\cdot,\cdot)$. We define a mapping from $\cP$ to $\cH$ such that each distribution $p\in \cP$ is mapped into an element in $\cH$ as follows
\[\mu_p(\cdot)=\mE_p [k(\cdot,x)]=\int k(\cdot,x)dp(x). \]
Here, $\mu_p(\cdot)$ is referred to as the {\em mean embedding} of the distribution $p$ into the Hilbert space $\cH$. Due to the reproducing property of $\cH$, it is clear that $\mE_p[f]=\langle \mu_p,f \rangle_{\cH}$ for all $f \in \cH$.

It is desirable that the embedding is {\em injective} such that each $p \in \cP$ is mapped to a unique element $\mu_p \in \cH$. It has been shown in \cite{Fuku2008,Srip2008,Fuku2009,Srip2010} that for many RKHSs such as those associated with Gaussian and Laplace kernels, the mean embedding is injective.
In this way, many machine learning problems with unknown distributions can be solved by studying mean embeddings of probability distributions without actually estimating the distributions, e.g., \cite{Song2013,Song2011a,Song2011b,Smola2007}.  In order to distinguish between two distributions $p$ and $q$, \cite{Gretton2012} introduced the following quantity of maximum mean discrepancy (MMD) based on the mean embeddings $\mu_p$ and $\mu_q$ of $p$ and $q$, respectively:
\begin{equation}
\mmd[p,q]:=\|\mu_p-\mu_q\|_{\cH}.
\end{equation}
It is also shown that
\[\mmd[p,q]=\sup_{\|f\|_{\cH} \leq 1} \mE_p[f(x)]-\mE_q[f(x)].\]
Namely, $\mmd[p,q]$ achieves the maximum of the mean difference of a function between the two distributions over all unit-norm functions in the RKHS $\cH$.

Due to the reproducing property of kernel, it can be easily shown that
\begin{flalign}\label{eq:mmdpq}
\mmd^2[p,q]=&\mE_{x,x'}[k(x,x')]-2\mE_{x,y}[k(x,y)] +\mE_{y,y'}[k(y,y')],
\end{flalign}
where $x$ and $x'$ are independent and  have the same distribution $p$, and $y$ and $y'$ are independent and have the same distribution $q$. An unbiased estimator of $\mmd^2[p,q]$ based on $n$ samples of $x$ and $m$ samples of $y$ is given by
\begin{flalign}\label{eq:mmdu}
 \mmd_u^2[X,Y]=&\frac{1}{n(n-1)}\sum_{i=1}^n\sum_{j\neq i}^n k(x_i,x_j) +\frac{1}{m(m-1)}\sum_{i=1}^m\sum_{j\neq i}^m k(y_i,y_j)\nn\\&-\frac{2}{nm}\sum_{i=1}^n\sum_{j=1}^m k(x_i,y_j),
\end{flalign}
where $X=[x_1,\ldots,x_n]$, and $Y=[y_1,\ldots,y_m]$.
We note that other estimators of   $\mmd^2[p,q]$ are also available, which can be used for our problem. In this paper, we focus on the unbiased estimator given above to convey the central idea.


\section{Line Network}\label{sec:line}

\subsection{Test and Performance}\label{sec:suff}

We construct a nonparametric test using the unbiased estimator in \eqref{eq:mmdu} and the scan statistics. For each interval $I$, let $Y_I$ denote the samples in the interval $I$, and $Y_{\bar I}$ denote the samples outside the interval $I$. We compute $\mmd^2_{u,I}(Y_I,Y_{\bar{I}})$ for all intervals $I\in \cI_n^{(a)}$. Under the null hypothesis $H_0$, all  samples are generated from the distribution $p$. Hence, for each $I\in \cI_n^{(a)}$, $\mmd^2_{u,I}(Y_I,Y_{\bar{I}})$
yields an estimate of $\mmd^2[p,p]$, which is zero. Under the alternative hypothesis $H_1$, there exists an anomalous interval $I^*$ in which the samples are generated from distribution $q$. Hence, $\mmd^2_{u,I^*}(Y_{I^*},Y_{\bar{I^*}})$ yields an estimate of $\mmd^2[p,q]$, which is bounded away from zero due to the fact that $p\neq q$. Based on the above understanding, we build the following test:
\begin{flalign}\label{eq:test}
\max_{I:I\in \cI_n^{(a)}} \mmd^2_{u,I}(Y_I,Y_{\bar{I}})
\begin{cases}
\ge t, \quad &\text{determine } H_1 \\
< t, &  \text{determine } H_0
\end{cases}
\end{flalign}
where $t$ is a threshold parameter. It is anticipated that with sufficiently accurate estimate of MMD and an appropriate choice of the threshold $t$, the test in \eqref{eq:test} should provide desired performance. The following theorem characterizes the performance of the proposed test.
\begin{theorem}\label{thm:achiv}
Suppose the test in \eqref{eq:test} is applied to the nonparametric problem described in Section \ref{sec:model}. Further assume that the kernel in the test satisfies $0\leq k(x,y)\leq K$ for all $(x,y)$. Then, the type I and type II errors are upper bounded respectively as follows:
  \begin{flalign}
    P(H_1|H_0)&\leq \sum_{I:I_{\min}\leq |I|\leq I_{\max}}\exp\left(-\frac{t^2|I|(n-|I|)}{8K^2n}\right) \nn\\
    &= \sum_{I_{\min}\leq i\leq I_{\max}}(n-i+1)\exp\left(-\frac{t^2i(n-i)}{8K^2n}\right),\label{eq:type1error}\\
    P(H_0|H_{1,I})&\leq  \exp\left(-\frac{(\mmd^2[p,q]-t)^2|I|(n-|I|)}{8nK^2}\right), \text{ for }I\in\mathcal I_n^{(a)}\label{eq:type2error}
  \end{flalign}
  where $t$ is the threshold of the test that satisfies $t< \mmd^2[p,q]$.

  Furthermore, the test \eqref{eq:test} is consistent if
  \begin{flalign}
    I_{\min}&\geq  \frac{16K^2(1+\eta)}{t^2}\log n,\label{eq:suf}\\
     I_{\max}&\leq n-\frac{16K^2(1+\eta)}{t^2}\underset{\text{arbitrary k number of }\log}{\underbrace{\log\cdots\cdots\log }n},\label{eq:suf2}
  \end{flalign}
  where $\eta$ is any positive constant.
\end{theorem}
\begin{proof}
  See Appendix \ref{app:line}.
\end{proof}

We note that many kernels satisfy the boundedness condition required in Theorem \ref{thm:achiv}, such as Gaussian kernel and Laplacian Kernel.

The above theorem implies that to guarantee consistency of the proposed test, the minimum length $I_{\min}$ should scale at the order $I_{\min}=\Omega(\log n)$. Furthermore, $n-I_{\max}$ should scale at the order $\Omega(\underset{\text{arbitrary k number of }\log}{\underbrace{\log\cdots\cdots\log }n})$  which can be arbitrarily slow. Hence, the number of candidate anomalous intervals in the set $\mathcal I_n^{(a)}$ is  $\Theta(n^2)$, which is at the same order as the number of all intervals. Hence, at the order sense, not many intervals are excluded from being anomalous.

It can be seen that the conditions \eqref{eq:suf} and \eqref{eq:suf2} on $I_{\min}$ and $I_{\max}$ are asymmetric. This can be understood by the upper bound \eqref{eq:type1error} on the type I error, which is a sum over all candidate anomalous intervals with length between $I_{\min}$ and $I_{\max}$. Due to the specific geometric structure of the line network, as the length $|I|$   increases, the number of candidate anomalous intervals with length $|I|$ equals $n-|I|+1$ and decreases as $|I|$ increases. Although the term $\exp\left(-\frac{t^2i(n-i)}{8K^2n}\right)$ in \eqref{eq:type1error} is symmetric over $i$ with respect to $\frac{n}{2}$, the entire term $(n-i+1)\exp\left(-\frac{t^2i(n-i)}{8K^2n}\right)$ is not symmetric, which consequently yields the asymmetric conditions on $I_{\min}$ and $I_{\max}$.

Theorem \ref{thm:achiv} requires that the threshold $t$ in the test \eqref{eq:test} to be less than $\mmd^2[p,q]$. In practice, the information of $\mmd^2[p,q]$ may or may not be available depending on specific applications. If it is known, then the threshold $t$ can be set as a constant smaller than $\mmd^2[p,q]$. If it is unknown, then the threshold $t$ needs to scale to zero as $n$ gets large in order to be asymptotically smaller than $\mmd^2[p,q]$. We summarize these two cases in the following two corollaries.


\begin{corollary}\label{cor:t1}
  If the value $\mmd^2[p,q]$ is known a priori, we set the threshold $t=(1-\delta)\mmd^2[p,q]$ for any $0< \delta < 1$. The test in \eqref{eq:test} is consistent, if  $I_{\min}$ and $I_{\max}$ satisfy the following conditions,
  \begin{flalign}\label{c1}
    I_{\min}&\geq\frac{16K^2(1+\eta')}{\mmd^4[p,q]}\log n\nn\\
    I_{\max}&\leq n-\frac{16K^2(1+\eta')}{\mmd^4[p,q]}\underset{\text{arbitrary k number of }\log}{\underbrace{\log\cdots\cdots\log }n},
  \end{flalign}
  where $\eta'$ is any positive constant.
\end{corollary}
Corollary \ref{cor:t1} follows directly from Theorem \ref{thm:achiv} by setting $\eta'=\frac{1+\eta}{(1-\delta)^2}-1$.

\begin{corollary}\label{crl:unknownt}
  If the value $\mmd^2[p,q]$ is unknown, we set the threshold t to scale with $n$, such that $\lim_{n\rightarrow \infty}t_n=0$. The test in \eqref{eq:test} is consistent, if  $I_{\min}$ and $I_{\max}$ satisfy the following conditions,
  \begin{flalign}\label{c2}
    I_{\min}&\geq\frac{16K^2(1+\eta)}{t_n^2}\log n\nn\\
    I_{\max}&\leq n-\frac{16K^2(1+\eta)}{t_n^2}\underset{\text{arbitrary k number of }\log}{\underbrace{\log\cdots\cdots\log }n},
  \end{flalign}
  where $\eta$ is any positive constant.
\end{corollary}
Corollary \ref{crl:unknownt} follows directly from Theorem \ref{thm:achiv} by noting that $t_n < \mmd^2[p,q]$ for $n$ large enough.

We note that Corollary \ref{crl:unknownt} holds for any $t_n$ that satisfies $\lim_{n\rightarrow \infty}t_n = 0$. It is clear from Corollary \ref{crl:unknownt} that for the case when $\mmd^2[p,q]$ is unknown,  $I_{\min}$ should scale at the order  $\omega(\log n)$, and $n-I_{\max}$ should scale at the order $\omega(\underset{\text{arbitrary k number of }\log}{\underbrace{\log\cdots\cdots\log }n})$.
Hence, comparison of the above two corollaries implies that the prior knowledge about $\mmd^2[p,q]$ is very important for network ability to identifying anomalous events. If $\mmd^2[p,q]$ is known, then the network can resolve an anomalous object with the size $\Omega(\log n)$. However, if such knowledge is unknown, the network can resolve only bigger anomalous objects with the size $\omega(\log n)$.


We note that Theorem \ref{thm:achiv} and Corollaries \ref{cor:t1} and \ref{crl:unknownt} characterize the conditions to guarantee test consistency for a pair of fixed but unknown distributions $p$ and $q$. Hence, the conditions \eqref{eq:suf}, \eqref{eq:suf2}, \eqref{c1} and \eqref{c2} depend on the underlying distributions $p$ and $q$. In fact, these conditions further yield the following condition that guarantees the proposed test to be universally consistent for any arbitrary $p$ and $q$.
\begin{proposition}[Universal Consistency]\label{prop:universal}
Consider the nonparametric problem given in Section \ref{sec:model}. Further assume the test in \eqref{eq:test} applies a bounded kernel  with $0\leq k(x,y)\leq K$ for any $(x,y)$. Then the test \eqref{eq:test} is universally consistent for any arbitrary pair of $p$ and $q$, if
  \begin{flalign}\label{c3}
    I_{\min}&=\omega(\log n)\nn\\
    I_{\max}&=n-\Omega(\underset{\text{arbitrary k number of }\log}{\underbrace{\log\cdots\cdots\log }n})
  \end{flalign}
\end{proposition}
\begin{proof}
This result follows from \eqref{eq:suf}, \eqref{eq:suf2}, \eqref{c1} and \eqref{c2} and the fact that $\mmd[p,q]$ is a constant for any given $p$ and $q$.
\end{proof}

\subsection{Necessary Conditions}
In Section \ref{sec:suff}, Proposition \ref{prop:universal} suggests  the sufficient conditions on $I_{\min}$ and $I_{\max}$ to guarantee  the proposed nonparametric test to be universally consistent for arbitrary $p$ and $q$.
In the following theorem, we characterize the necessary conditions on $I_{\min}$ and $I_{\max}$ that any test must satisfy in order to be universally consistent for arbitrary $p$ and $q$.
\begin{theorem}\label{thm:conv}
  For the nonparametric detection problem described in Section \ref{sec:model} over a line network, any test must satisfy the following conditions on $I_{\min}$ and $I_{\max}$ in order to be universally consistent for arbitrary $p$ and $q$:
  \begin{flalign}\label{eq:nec}
    I_{\min}&=\omega(\log n)\nn\\
    \text{and }n-I_{\max}&\rightarrow\infty, \text{ as } n\rightarrow \infty.
  \end{flalign}
\end{theorem}
\begin{proof}
  See Appendix \ref{proof:conv}. The idea of the proof is to first lower bound the minimax risk by the Bayes risk of a simpler problem. Then for such a problem, we show that there exist $p$ and $q$ (in fact Gaussian $p$ and $q$) such that even the optimal parametric test is not consistent under the conditions given in the theorem. This thus implies that under the same condition, no nonparametric test is universally consistent for arbitrary $p$ and $q$.
  \end{proof}

It can be seen that the necessary condition on $I_{\min}$ in \eqref{eq:nec} matches   the sufficient condition in \eqref{c3} at the order level which implies that  the proposed test is order-level optimal in $I_{\min}$. Furthermore,  the sufficient condition on $I_{\max}$ can arbitrarily slowly converge to $n$, which is also very close to the necessary condition on $I_{\max}$. Thus we have the following theorem.
\begin{theorem}[Optimality]
Consider the nonparametric detection problem described in Section \ref{sec:model}.  The MMD-based test \eqref{eq:test} is order-level optimal in terms of $I_{\min}$ and nearly order-level optimal in terms of $I_{\max}$ required to guarantee universal  test consistency for arbitrary $p$ and $q$.
\end{theorem}

\section{Generalization to Other Networks}\label{sec:generalization}

In this section, we generalize our study to three other networks in order to demonstrate more generality of our approach. For each network, our study further demonstrates how the geometric property of the network affects the conditions required to guarantee the test consistency.

\subsection{Detecting Interval in Ring Networks}\label{sec:ring}

In this subsection, we consider a ring network (see Figure~\ref{fig:model2}), in which $n$ nodes are located over a ring. We define an interval $I$ to be a subset of consecutive nodes over the ring.  We consider the  following set of candidate anomalous intervals,
\begin{flalign}
  \mathcal I_n^{(a)}=\{I:I_{\min}\leq |I|\leq I_{\max}\},
\end{flalign}
where $I_{\min}$ and $I_{\max}$ are minimal and maximal lengths of all candidate anomalous intervals.
Despite similarities that the ring network shares with the line network, its major difference lies in that
the number of candidate anomalous intervals with size $k$ is $n$ (which remains the same as $k$ increases) as opposed to $n-k+1$ in the line network (which decreases as $k$ increases). Consequently, the number of sub-hypotheses in $H_1$ is different. Such difference is reflected in the results that we present next.

\begin{figure}[htb]
\begin{center}
\includegraphics[width=5cm]{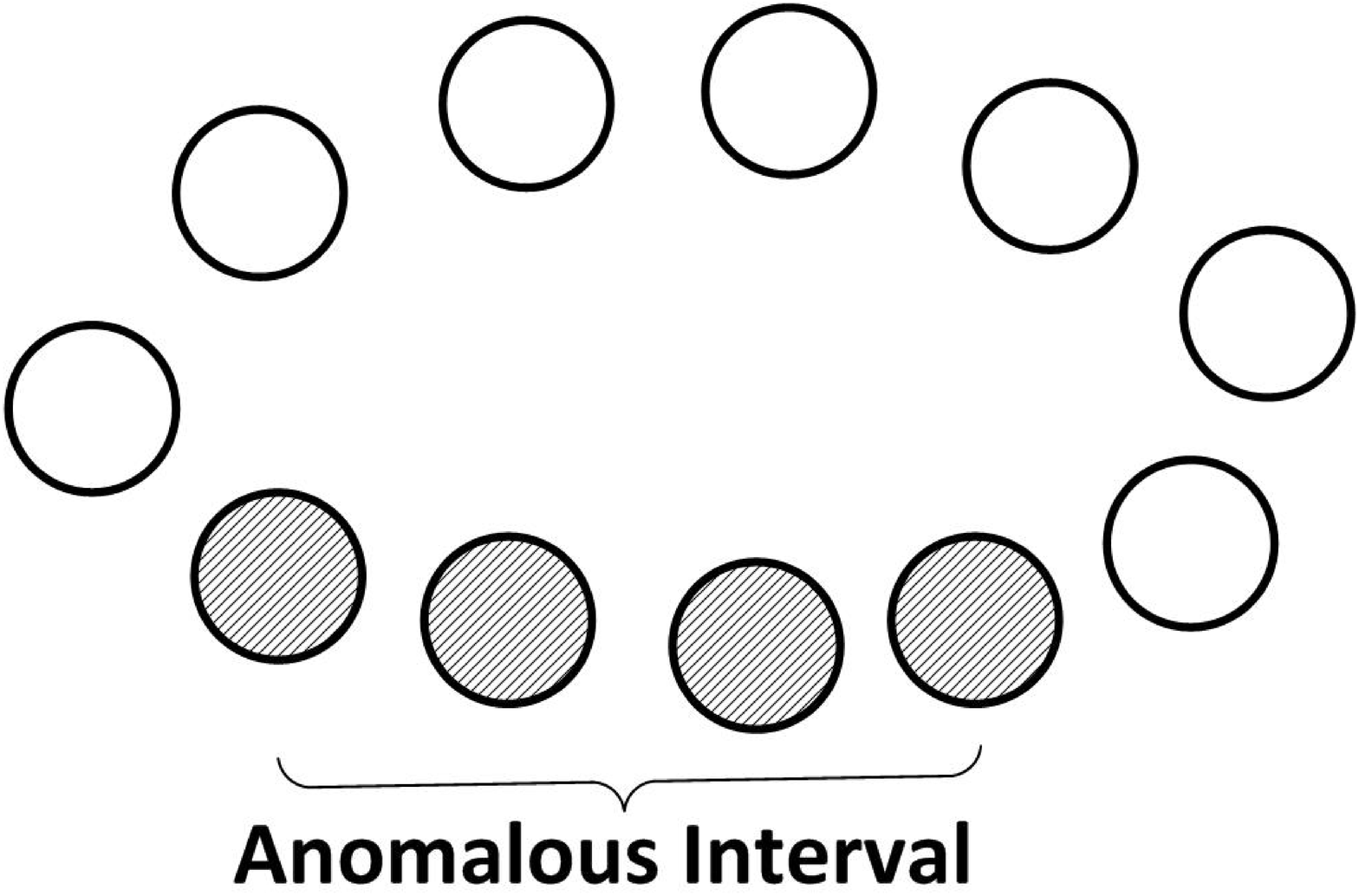}
\caption{A ring  network with an anomalous interval}\label{fig:model2}
\end{center}
\end{figure}

We construct the test  as follows:
\begin{flalign}\label{eq:test2}
\max_{I:I\in \cI_n^{(a)}} \mmd^2_{u,I}(Y_I,Y_{\bar{I}})
\begin{cases}
\ge t, \quad &\text{determine } H_1 \\
< t, &  \text{determine } H_0
\end{cases}
\end{flalign}
where $Y_I$ denotes the samples in the interval $I$, $Y_{\bar I}$ denotes the samples outside the interval $I$, and $t$ is a threshold parameter.

If we apply the test \eqref{eq:test2} with a bounded kernel, then it can be shown (see Appendix \ref{proof:ring}) that the type I and type II errors are bounded as follows:
  \begin{flalign}
    P(H_1|H_0)&\leq \exp{\left(2\log n-\frac{2t^2\min\{I_{\min}(n-I_{\min}),I_{\max}(n-I_{\max})\}}{16nK^2}\right)},\\
    P(H_0|H_{1,I})&\leq \exp\left(-\frac{(\mmd^2[p,q]-t)^2|I|(n-|I|)}{8nK^2}\right),\text{ for }I\in\mathcal I_n^{(a)}
  \end{flalign}
  where $t$ is the threshold of the test that satisfies $t< \mmd^2[p,q]$. Furthermore, the test in \eqref{eq:test2} is consistent, if
  \begin{flalign}
    I_{\min}&\geq\frac{16K^2(1+\eta)}{t^2}\log n, \label{eq:iminring}\\
    I_{\max}&\leq n-\frac{16K^2(1+\eta)}{t^2}\log n, \label{eq:imaxring}
  \end{flalign}
  where $\eta$ is any positive constant.
The detailed proof can be found in Appendix \ref{proof:ring}.
Comparing the above conditions with Theorem \ref{thm:achiv} suggests that although the sufficient conditions on $I_{\min}$ are the same, the conditions on $I_{\max}$ reflect  order-level difference in line and ring networks. For line networks, an anomalous interval can be close to the entire network with only a gap of length $\Omega(\underset{\text{arbitrary k number of }\log}{\underbrace{\log\cdots\cdots\log }n})$. However, for ring networks, the gap can be as large as $\Omega(\log n)$. Such difference in tests' resolution of anomalous intervals is mainly due to the difference in network geometry that further affects the error probability of  tests. By carefully comparing the two types of errors, in fact, the type II error converges to zero as the network size goes to infinity as long as the number of anomalous samples (i.e., length of anomalous intervals) and the number of typical samples (i.e., the gap between anomalous intervals and the entire network) both scale with $n$ to infinity. Thus, the conditions for the type II error being asymptotically small are the same for the two types of networks. The situation is different for the type I error. The key observation is that the number of candidate anomalous intervals with size $k$ is $n-k+1$ in a line network (which decreases as $k$ increases), but is $n$ in a ring network (which remains the same as $k$ increases). Such difference can be as significant as the order level if $k$ is close to $n$, say $n-\Omega(\log n)$. Consequently, the type I error for a line network can be much smaller than that for a ring network, resulting more relaxed condition on $I_{\max}$ to guarantee  consistency.

Similarly to the line network, setting the threshold $t$ for the test \eqref{eq:test2} can be considered in two cases with and without the information of $\mmd[p,q]$. If $\mmd[p,q]$ is known, set $t=(1-\delta)\mmd^2[p,q]$. Otherwise, $t$ can be chosen to scale to zero as $n$ goes to infinity. Similar results as in Corollary \ref{cor:t1} and Corollary \ref{crl:unknownt} can then be derived for a ring network.

Furthermore, \eqref{eq:iminring} and \eqref{eq:imaxring} imply that the test \eqref{eq:test2} is {\em universally consistent} for any arbitrary $p$ and $q$, if
 \begin{flalign}
    I_{\min}=\omega(\log n), \quad \text{and} \quad n-I_{\max} =\omega(\log n).
  \end{flalign}

Following the arguments similar to those for the line network, it can be shown that any test must satisfy the following necessary conditions required on $I_{\min}$ and $I_{\max}$ in order to be {\em universally consistent} for arbitrary $p$ and $q$:
  \begin{flalign}\label{eq:ringconv}
    I_{\min}&=\omega(\log n),\quad \text{ and}\quad n-I_{\max}\rightarrow \infty \text{, as } n\rightarrow \infty.
  \end{flalign}
The detailed proof can be found in Appendix \ref{proof:conv_ring}.

Therefore, combining the above sufficient and necessary conditions, we conclude the following optimality property for the proposed test.
\begin{theorem}[Optimality]
Consider the problem of nonparametric detection of an interval over a ring network. The MMD-based test \eqref{eq:test2} is order-level optimal in terms of $I_{\min}$ required to guarantee universal test consistency for arbitrary $p$
and $q$.
\end{theorem}

\subsection{Detecting Disk in Two-Dimensional Lattice Network}\label{sec:disk}
\begin{figure}
\begin{center}
\includegraphics[width=6.5cm]{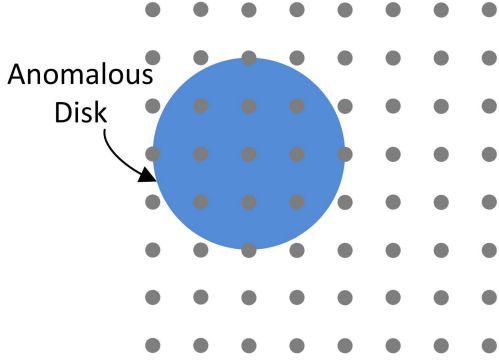}
\caption{Two-dimensional lattice network with an anomalous disk}\label{fig:disk}
\end{center}
\end{figure}
We consider a two-dimensional lattice network (see Figure \ref{fig:disk}) consisting of  $n^2$ nodes placed at the corner points of a lattice. Consider the following set of candidate anomalous disks with each disk centered at a certain node with integer radius:
\begin{flalign}
 \mathcal D_n^{(a)}=\{D: D_{\min}\leq |D|\leq D_{\max}\},
\end{flalign}
where $|D|$ denotes the number of nodes within the disk $D$, $D_{\min}:=\min_{D\in\mathcal D_n^{(a)}}|D|$ and $D_{\max}:=\max_{D\in\mathcal D_n^{(a)}}|D|$.
The goal is to detect the existence of an anomalous disk over the  lattice network.
Towards this end, we build the following test:
\begin{flalign}\label{eq:test3}
\max_{D:D\in \mathcal D_n^{(a)}} \mmd^2_{u,D}(Y_D,Y_{\overline{D}})
\begin{cases}
\ge t, \quad &\text{determine } H_1 \\
< t, &  \text{determine } H_0
\end{cases}
\end{flalign}
where $Y_D$ contains samples within the disk $D$, and $Y_{\overline D}$ contains samples outside the disk $D$.
If we apply this test with a bounded kernel, then the type I error can be bounded as follows:
  \begin{flalign}
    P(H_1|H_0)\leq \exp{\left(3\log n-\frac{2t^2\min\{D_{\min}(n^2-D_{\min}),D_{\max}(n^2-D_{\max})\}}{16n^2K^2}\right)},
  \end{flalign}
 and the type II error can be bounded as follows:
  \begin{flalign}
    P(H_0|H_{1,D})\leq \exp\left({-\frac{(\mmd^2[p,q]-t)^2|D|(n^2-|D|)}{8n^2K^2}}\right) \text{ for }D\in\mathcal D^{(a)},
  \end{flalign}
  where $t$ is the threshold of the test that satisfies $t< \mmd^2[p,q]$.

It can be further shown that if $D_{\min}$ and $D_{\max}$ satisfy the following conditions:
  \begin{flalign}
    D_{\min}&\geq\frac{24K^2(1+\eta)}{t^2}\log n, \label{eq:dmin}\\
    D_{\max}&\leq n^2-\frac{24K^2(1+\eta)}{t^2}\log n,\label{eq:dmax}
  \end{flalign}
where $\eta$ is any positive constant, then the test \eqref{eq:test3} is consistent. Interestingly, the largest disk within a two-dimensional lattice network has radius to be $\frac{n}{2}$ and areas to be $\frac{\pi n^2}{4}\approx 0.79 n^2$, which contains at most $cn^2$ nodes with $c<1$ for large $n$. This implies that  the bound on $D_{\max}$ in \eqref{eq:dmax} is satisfied automatically when $n$ is large.

Hence, for large $n$, \eqref{eq:dmin} implies that the test \eqref{eq:test3} is {\em universally consistent} for any arbitrary $p$ and $q$, if
 \begin{flalign}
    D_{\min}=\omega(\log n).
  \end{flalign}
Furthermore, following the arguments similar to those for the line network, it can be shown that any test must satisfy the following necessary condition required on $D_{\min}$ in order to be {\em universally consistent} for arbitrary $p$ and $q$:
\begin{flalign}
    D_{\min}=\omega(\log n).
  \end{flalign}
Therefore, combining the above sufficient and necessary conditions, we conclude the following optimality property for the proposed test.
\begin{theorem}[Optimality]
Consider the problem of nonparametric detection of a disk over two-dimensional lattice network. The MMD-based test \eqref{eq:test3} is order-level optimal in the size of disks required to guarantee universal test consistency for arbitrary $p$
and $q$.
\end{theorem}


%

\subsection{Detecting Rectangle in Lattice Networks}

We consider a r-dimensional lattice network consisting of $n^r$ nodes placed at the corner points of a lattice network. Consider the following set of candidate anomalous rectangles:
\[S_n^{(a)}:=\{S=[I_1\times I_2\times \ldots\times I_r]: S_{\min}\leq |S|\leq S_{\max}\},\]
where $I_i$ for $1\leq i\leq r$ denotes an interval contained in  $[1,n]$ with consecutive indices,
$|S|$ denotes the number of nodes in the rectangle $S$, $S_{\min}:=\min_{S\in\mathcal S_n^{(a)}}|S|$, and $S_{\max}:=\max_{S\in\mathcal S_n^{(a)}}|S|$.
The goal is to detect existence of an anomalous r-dimensional rectangle.
Towards this end, we build the following test,
\begin{flalign}\label{eq:test4}
\max_{S:S\in \mathcal S_n^{(a)}} \mmd^2_{u,S}(Y_S,Y_{\overline{S}})
\begin{cases}
\ge t, \quad &\text{determine } H_1 \\
< t, &  \text{determine } H_0
\end{cases}
\end{flalign}
where $Y_S$ contains samples within the rectangular $S$, and $Y_{\overline S}$ contains samples outside the rectangular $S$. If we apply this test with a bounded kernel, then the type I error is bounded as follows:
  \begin{flalign}
    P(H_1|H_0)\leq \exp{\left(2r\log n-\frac{2t^2\min\{S_{\min}(n^r-S_{\min}),S_{\max}(n^r-S_{\max})\}}{16n^rK^2}\right)},
  \end{flalign}
and the type II error is bounded as follows:
  \begin{flalign}
    P(H_0|H_{1,S})\leq \exp{\left(-\frac{(\mmd^2[p,q]-t)^2|S|(n^r-|S|)}{8n^rK^2}\right)}, \text{ for }S\in\mathcal S^{(a)}
  \end{flalign}
  where $t$ is the threshold of the test that satisfies $t< \mmd^2[p,q]$.

It can be further shown that if $S_{\min}$ and $S_{\max}$  satisfy the following conditions:
  \begin{flalign}
    S_{\min}&\geq\frac{16rK^2(1+\eta)}{t^2}\log n \label{eq:smin}\\
    S_{\max}&\leq n^r-\frac{16rK^2(1+\eta)}{t^2}\log n,\label{eq:smax}
  \end{flalign}
where $\eta$ is any positive constant, then the test in \eqref{eq:test4} is consistent.

We here note an important fact that as long as the largest anomalous rectangle is not the entire lattice network, it can at most contain $n^r-n^{r-1}$ nodes, which satisfies the condition \eqref{eq:smax} for large $n$ as well as the following condition
\begin{flalign}\label{eq:smax1}
n^r-S_{\max}\rightarrow\infty \; \text{ as } n\rightarrow \infty.
\end{flalign}
Consequently, \eqref{eq:smax} and \eqref{eq:smax1} are equivalent, both requiring the largest anomalous rectangle not to be the entire network. Thus, the conditions \eqref{eq:smin} and \eqref{eq:smax1} imply that the test \eqref{eq:test4} is {\em universally consistent} for any arbitrary $p$ and $q$, if
 \begin{flalign}
    S_{\min}=\omega(\log n), \quad \text{and}\quad n^r-S_{\max}&\rightarrow\infty \text{ as }n\rightarrow \infty.
  \end{flalign}



Furthermore, following the arguments similar to those for the line network, it can be shown that any test must satisfy the following necessary conditions required on $S_{\min}$ and $S_{\max}$ in order to be {\em universally consistent} for arbitrary $p$ and $q$:
\begin{flalign}
    S_{\min}=\omega(\log n),\quad \text{and}\quad n^r-S_{\max}&\rightarrow\infty \text{ as }n\rightarrow \infty.
  \end{flalign}
Therefore, combining the above sufficient and necessary conditions, we conclude the following optimality property for the proposed test.
\begin{theorem}[Optimality]
Consider the problem of nonparametric detection of a rectangle over a lattice network. The MMD-based test \eqref{eq:test4} is order-level optimal to guarantee universal test consistency for arbitrary $p$ and $q$.
\end{theorem}


\section{Numerical Results}\label{sec:num}
In this section, we provide numerical results to demonstrate the performance of our tests and compare our approach with other competitive approaches.

\begin{figure}
\begin{minipage}[t]{0.45\linewidth}
\centering
  \includegraphics[width=2.5in]{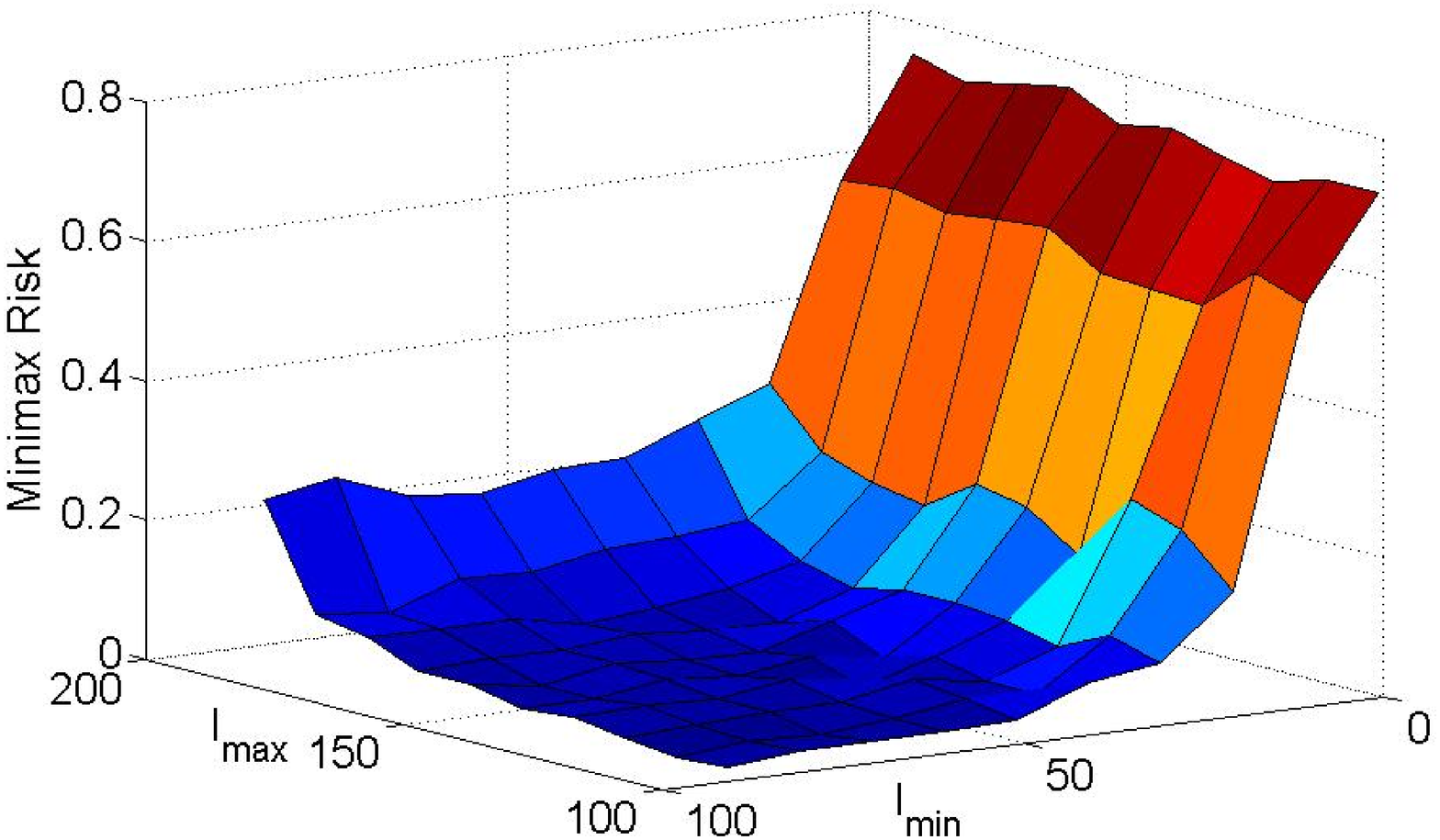}\\
  \caption{Minimax risk for a line network}\label{fig:line_smlt}
\end{minipage}
\hfill
\begin{minipage}[t]{0.45\linewidth}
\centering
  \includegraphics[width=2.5in]{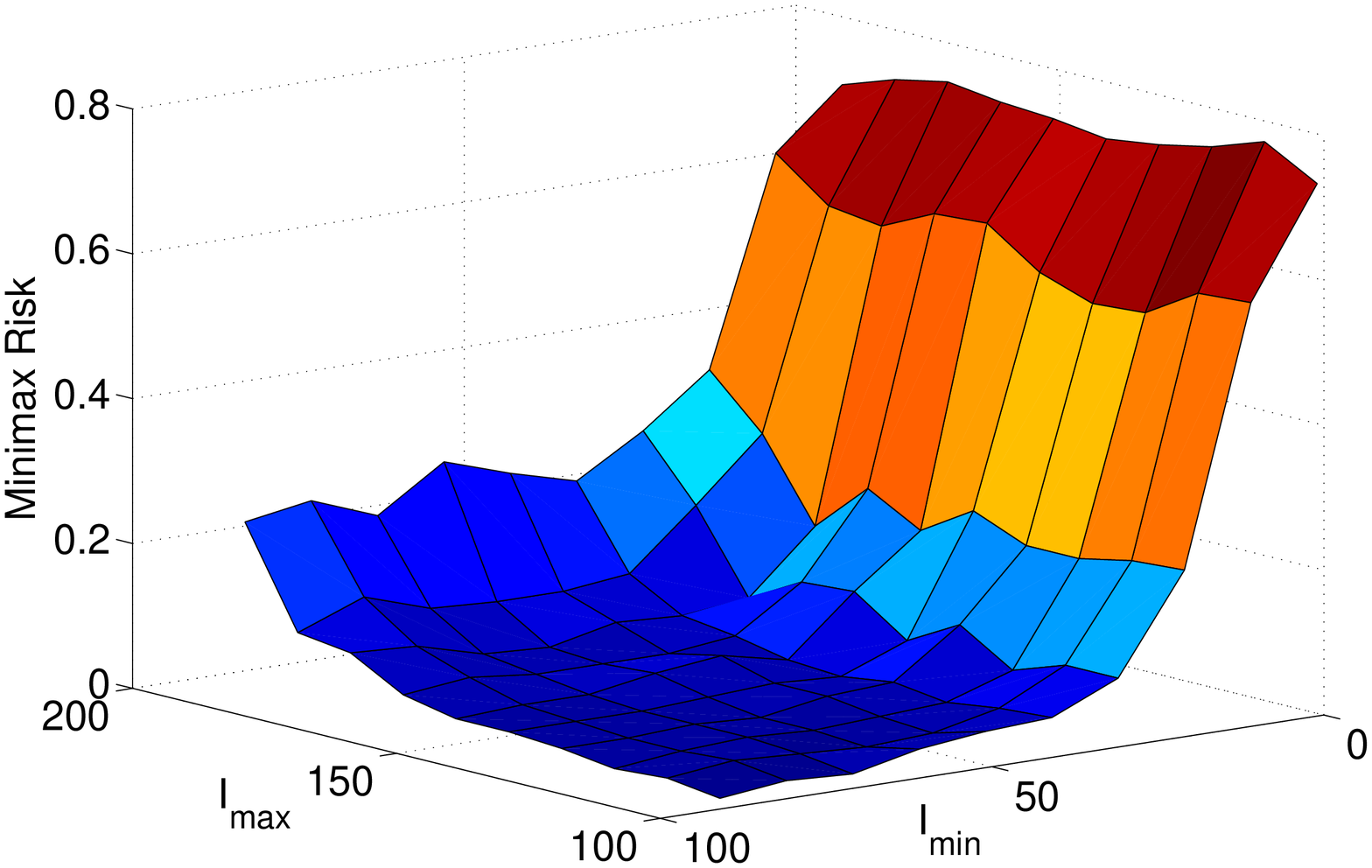}\\
  \caption{Minimax risk for a ring network}\label{fig:ring_smlt}
  \end{minipage}
\end{figure}
\begin{table}[htb]
  \centering
  \caption{Minimax risk for a line network}\label{table:1}
  \begin{tabular}{c|c|c|c|c|c}
    \hline
    $I_{\min}\setminus I_{\max}$ &100    &110   &130   &160    &190\\
    \hline
    1   &0.73   &0.73   &0.72  &0.76    &0.76\\
    11  &0.58   &0.61   &0.55   &0.59   &0.59\\
    31  &0.09   &0.11   &0.10   &0.09   &0.27\\
    61  &0.03   &0.03   &0.05   &0.06   &0.21\\
    91  &0.02   &0.02   &0.04   &0.05   &0.24\\
    \hline
  \end{tabular}
\end{table}
\begin{table}[htb]
  \centering
  \caption{Minimax risk for a ring network}\label{table:2}
  \begin{tabular}{c|c|c|c|c|c}
    \hline
    $I_{\min}\setminus I_{\max}$ &100    &110   &130   &160    &190\\
    \hline
    1   &0.73   &0.78   &0.74  &0.74    &0.71\\
    11  &0.58   &0.61   &0.53   &0.60   &0.63\\
    31  &0.09   &0.10   &0.12   &0.12   &0.27\\
    61  &0.03   &0.03   &0.03   &0.05   &0.27\\
    91  &0.02   &0.03   &0.03   &0.05   &0.24\\
    \hline
  \end{tabular}
\end{table}

In the first experiment, we apply our test to the line and ring networks.
We set the network size $n=200$, the distribution $p$ to be Gaussian with mean zero and variance one, and the anomalous distribution $q$ to be Gaussian
 with mean one and variance one. We use Gaussian kernel with $\sigma=1$. In Figures \ref{fig:line_smlt} and \ref{fig:ring_smlt}, we plot the minimax risk (normalized by 2) for line and ring networks as functions of $I_{\min}$ and $I_{\max}$. For further illustration, we also  list some values of the two risk functions in Tables \ref{table:1} and \ref{table:2}.
It can be seen from Tables  \ref{table:1} and \ref{table:2}, and Figures \ref{fig:line_smlt} and \ref{fig:ring_smlt} that the risk functions decrease as $I_{\min}$ increases and as $I_{\max}$ decreases. This is reasonable because as $I_{\min}$ increases and as $I_{\max}$ decreases, the number of candidate anomalous intervals decreases, which reduces   the difficulty of detection. The minimum numbers of samples inside and outside the anomalous interval also increase, respectively, which provide more accurate information about the distributions.

In the next experiment, we compare the performance of our test with other competitive tests including the student t-test, the Smirnov test \cite{Fried1979}, the Wolf test \cite{Fried1979}, Hall test \cite{Hall2002}, kernel-based KFDA test \cite{Harchaoui2008} and kernel-based KDR test \cite{Sugiyama2012}.
We consider a line network with the network size $n=100$. We set the distribution $p$ to be Gaussian with zero mean and variance $2$, and set the anomalous distribution $q$ to be a mixture of Gaussian distributions with equal probability taking $\mathcal N(-1,1)$ and $\mathcal N(1,1)$. Hence, distributions $p$ and $q$ have the same mean and variance.
\begin{table}
  \centering
  \caption{Comparison of nonparametric approaches over a line network}\label{table:3}
  \begin{tabular}{cc|c|c|c|c}
    \hline
    $(I_{\min}, I_{\max})$  &t-test &Smirnov    &KFDA   &KDR        &MMD \\
    \hline
    (10,95)                 &0.90   &0.92       &0.70   &0.66       &\textbf{0.66}\\
    (10,50)                 &0.88   &0.90       &0.51   &0.56       &\textbf{0.55}\\
    (45,95)                 &0.89   &0.93       &0.54   &0.43       &\textbf{0.43}\\
    (45,50)                 &0.83   &0.62       &0.06   &0.06       &\textbf{0.05}  \\
    \hline
  \end{tabular}
\end{table}

In Table \ref{table:3}, we list some values of the risk function of our MMD-based test and other nonparametric tests with respect to various values  of $I_{\min}$ and $I_{\max}$.
It can be seen that the student t-test fails, because the test relies on difference in  mean and variance to distinguish two distributions, which are the same in our experiment. The Smirnov test    estimates the cumulative distribution function (CDF) first and then takes the maximum difference of the two cumulative distribution functions as the test statistics. For continuous distributions, accurately estimating the CDF from samples requires a large amount of data, which is not feasible  in our experiment.  For the three kernel-based tests KFDA, KDR and MMD, the performance are very close. In particular, for large enough $I_{\min}$  and small enough $I_{\max}$, the kernel-based tests yield small risk. Among these three kernel-based tests, MMD has a slightly better performance. In terms of the computational complexity, KFDA is much higher than KDR and MMD-based tests.
\section{Conclusion}\label{sec:con}

We have studied the nonparametric   problem of detecting the existence of an anomalous structure  over networks, in which both the typical and the anomalous distributions can be arbitrary and unknown.
We have developed  nonparametric tests  using the MMD to measure the distance between the mean embeddings of distributions into an RKHS.
We have analyzed the performance guarantee of our tests, and characterized the sufficient conditions on the minimum and maximum sizes of candidate anomalous structures to guarantee the consistency of our tests. We have further derived the necessary conditions  and showed that our tests are order-level optimal and nearly optimal respectively in terms of the minimum and maximum sizes of candidate structures.
We believe that the MMD-based approach can be applied to various detection problems involving distinguishing among distributions.

\hspace{3mm}

\appendix

\noindent {\Large \textbf{Appendix}}

\section{Proof of Theorem \ref{thm:achiv}}\label{app:line}
We first introduce the McDiarmid's inequality which is useful in bounding the probability of error in our proof.
\begin{lemma}[McDiarmid's Inequality]\label{mcdiarmid}
Let $f:\mathcal{X}^m\rightarrow \mathbb{R}$ be a function such that for all $i\in\{1,\ldots,m\}$, there exist $c_i<\infty$ for which
\begin{equation}\label{eq:mcc}
\underset{X\in\mathcal{X}^m, \tilde{x}\in \mathcal X}{sup}|f(x_1,\ldots,x_m)-f(x_1,\ldots x_{i-1},\tilde x,x_{i+1},\ldots,x_m)|\leq c_i.
\end{equation}
Then for any probability measure $P_X$ over $m$ independent random variables $X:=(X_1\ldots,X_m)$, and every $\epsilon>0$,
\begin{equation}
P_X\bigg(f(X)-E_X(f(X))>\epsilon\bigg)<\exp\left(-\frac{2\epsilon^2}{\sum_{i=1}^mc_i^2}\right),
\end{equation}
where $E_X$ denotes the expectation over $P_X$.
\end{lemma}

We now derive bounds on $P(H_1|H_0)$ and $P(H_0|H_{1,I})$ for the test \eqref{eq:test}. We first have
\begin{flalign}
  \mmd^2_{u,I}(Y_I,Y_{\bar{I}})=&\frac{1}{|I|(|I|-1)}\sum_{i\in I}\sum_{\substack{j\neq i\\j\in I}}k(y_i,y_j)+\frac{1}{(n-|I|)(n-|I|-1)}\sum_{i\notin I}\sum_{\substack{j\neq i\\j\notin I}}k(y_i,y_j)\nn\\
  &-\frac{2}{|I|(n-|I|)}\sum_{i\in I}\sum_{j\notin I}k(y_i,y_j).
\end{flalign}
Under $H_0$, all samples are generated from distribution $p$. Hence, $\mathbb{E}[\mmd^2_{u,I}(Y_I,Y_{\bar{I}})]=0$.

In order to apply the McDiarmid's inequality to bound the error probabilities $P(H_1|H_0)$ and $P(H_0|H_{1,I})$, we evaluate the following quantities. There are $n$ variables that affects the value of  $\mmd^2_{u,I}(Y_I,Y_{\bar{I}})$. We study the influence of these     $n$ variables on $\mmd^2_{u,I}(Y_I,Y_{\bar{I}})$ in the following two cases.
For $i\in I$, change of $y_i$ affects $\mmd^2_{u,I}(Y_I,Y_{\bar{I}})$ through the following terms,
\begin{flalign}
  \frac{2}{|I|(|I|-1)}\sum_{\substack{j\neq i\\j\in I}}k(y_i,y_j)-\frac{2}{|I|(n-|I|)}\sum_{j\notin I}k(y_i,y_j).
\end{flalign}
For $i\notin I$, change of $y_i$ affects $\mmd^2_{u,I}(Y_I,Y_{\bar{I}})$ through the following terms,
\begin{flalign}
  \frac{2}{(n-|I|)(n-|I|-1)}\sum_{\substack{j\neq i\\j\notin I}}k(y_i,y_j)-\frac{2}{|I|(n-|I|)}\sum_{j\in I}k(y_i,y_j).
\end{flalign}
Since the kernel we use is bounded, i.e., $0\leq k(x,y)\leq K$ for any $x$, $y$, we have that for $i\in I$, $c_i=\frac{4K}{|I|}$, and for $i\notin I$, $c_i=\frac{4K}{n-|I|}$, where $c_i$ serves the role as in \eqref{eq:mcc}.

Therefore, by applying McDiarmid's inequality, we obtain
\begin{flalign}
  P_{H_0}( \mmd^2_{u,I}(Y_I,Y_{\bar{I}})>t)
  \leq \exp{(-\frac{2t^2|I|(n-|I|)}{16nK^2})}.
\end{flalign}
Hence,
\begin{flalign}\label{eq:pp}
  P(H_1|H_0)&=P_{H_0}\left(\max_{I:I\in \cI_n^{(a)}} \mmd^2_{u,I}(Y_I,Y_{\bar{I}})>t\right)\nn\\
  &\overset{(a)}{\leq} \sum_{I:I\in \cI_n^{(a)}} P_{H_0}( \mmd^2_{u,I}(Y_I,Y_{\bar{I}})>t)\nn\\
  &\leq\sum_{I:I\in \cI_n^{(a)}}\exp{(-\frac{2t^2|I|(n-|I|)}{16nK^2})}\nn\\
  &=\sum_{I:I_{\min}\leq |I|\leq I_{\max}}\exp{(-\frac{2t^2|I|(n-|I|)}{16nK^2})}\nn\\
  &\overset{(b)}{=}\sum_{I:I_{\min}\leq |I|\leq n-\frac{16K^2(1+\eta)}{t^2}\log n}\exp{(-\frac{2t^2|I|(n-|I|)}{16nK^2})} \nn\\
&\quad  +\sum_{I:n-\frac{16K^2(1+\eta)}{t^2}\log n+1\leq |I|\leq I_{\max}}\exp{(-\frac{2t^2|I|(n-|I|)}{16nK^2})}
\end{flalign}
where (a) is due to Boole's inequality,  $\eta$
in (b) is a positive constant, and the second term in (b) is equal to zero if $n-\frac{16K^2(1+\eta)}{t^2}\log n+1\geq I_{\max}$.

It can be shown that if $I_{\min} \geq \frac{16K^2(1+\eta)}{t^2}\log n$, then
the first term in \eqref{eq:pp} can be bounded as follows,
\begin{flalign}\label{eq:ppp}
  \sum_{I:I_{\min}\leq |I|\leq n-\frac{16K^2(1+\eta)}{t^2}\log n}&\exp{(-\frac{2t^2|I|(n-|I|)}{16nK^2})}\nn\\
  &\overset{(a)}{\leq}n^2\exp{(-\frac{2t^2|I|(n-|I|)}{16nK^2})}\mid _{|I|=\frac{16K^2(1+\eta)}{t^2}\log n}\nn\\
  &=\exp{(-2\eta\log n+o(n))}\rightarrow 0, \text{ as } n\rightarrow \infty,
\end{flalign}
where (a) is due to the fact that there are at most $n^2$ number of candidate anomalous intervals contributing to the sum, and $|I|(n-|I|)$ is minimized by the value $|I|=\frac{16K^2(1+\eta)}{t^2}\log n$ within the range of $|I|$.

We next bound the second term in \eqref{eq:pp}.
\begin{flalign}
  &\sum_{n-\frac{16K^2(1+\eta)}{t^2}\log n+1\leq |I|\leq I_{\max}}\exp{\left(-\frac{2t^2|I|(n-|I|)}{16nK^2}\right)} \label{eq:482} \\
  &=\sum_{n-\frac{16K^2(1+\eta)}{t^2}\log n+1\leq |I|\leq n-\frac{16K^2(1+\eta)}{t^2}\log\log n}\exp{\left(-\frac{2t^2|I|(n-|I|)}{16nK^2}\right)}\label{eq:485}\\
  &\quad+\sum_{n-\frac{16K^2(1+\eta)}{t^2}\log\log n+1\leq |I|\leq I_{\max}}\exp{\left(-\frac{2t^2|I|(n-|I|)}{16nK^2}\right)}\nn\\
  &\leq \left(\frac{16K^2(1+\eta)}{t^2}\log n\right)^2\exp\left(-\frac{2t^2(n-\frac{16K^2(1+\eta)}{t^2}\log\log n)\frac{16K^2(1+\eta)}{t^2}\log\log n}{16nK^2}\right)\nn\\
  &\quad+\sum_{n-\frac{16K^2(1+\eta)}{t^2}\log\log n+1\leq |I|\leq I_{\max}}\exp{\left(-\frac{2t^2|I|(n-|I|)}{16nK^2}\right)},\label{eq:pp11}
\end{flalign}
where the first term in \eqref{eq:pp11} converges to zero as $n$ goes to infinity. The second term in \eqref{eq:pp11} can be bounded as
\begin{flalign}
  \sum_{n-\frac{16K^2(1+\eta)}{t^2}\log\log n+1\leq |I|\leq I_{\max}}&\exp{\left(-\frac{2t^2|I|(n-|I|)}{16nK^2}\right)}\nn\\
&\leq \left(\frac{16K^2(1+\eta)}{t^2}\log\log n \right)^2\exp \left( -\frac{2t^2I_{\max}(n-I_{\max})}{16nK^2}\right)\nn\\
\end{flalign}
which converges to zero as  $n\rightarrow \infty$ if
\begin{flalign}\label{eq:imax}
  I_{\max}\leq n-\frac{16K^2(1+\eta)}{t^2}\log\log\log n.
\end{flalign}
In fact, the condition \eqref{eq:imax} can be further relaxed by decomposing the second term in \eqref{eq:pp11} following  the steps similar to \eqref{eq:485} and \eqref{eq:pp11}. Such a procedure can be repeated for arbitrary finite times, say $k-2$ times, and it can be shown  that \eqref{eq:482} converges to zero  as $n\rightarrow \infty$ if
\begin{flalign}
  I_{\max}\leq n-\frac{16K^2(1+\eta)}{t^2}\underset{\text{arbitrary $k$ number of }\log}{\underbrace{\log\cdots\log\log}n}.
\end{flalign}

Therefore, we conclude  that the type I error, i.e., $P(H_1|H_0)$, converges to zero as $n\rightarrow \infty$ if the following conditions are satisfied:
\begin{flalign}\label{eq:51}
  I_{\min}&\geq\frac{16K^2(1+\eta)}{t^2}\log n\nn\\
  I_{\max}&\leq n-\frac{16K^2(1+\eta)}{t^2}\underset{\text{arbitrary $k$ number of }\log}{\underbrace{\log\cdots\log\log}n}
\end{flalign}
for any positive integer $k $.

We next continue to  bound the type II error $\underset{I\in \mathcal I_n^{(a)}}{\max}P(H_0|H_{1,I})$ as follows.
\begin{flalign}
  \underset{I\in \mathcal I_n^{(a)}}{\max}&P(H_0|H_{1,I})\nn\\
  &=\underset{I\in \mathcal I_n^{(a)}}{\max}P_{H_{1,I}}\left(\underset{I'\in\mathcal I_n^{(a)}}\max\mmd^2_{u,I'}(Y_{I'},Y_{\bar{I'}})<t\right)\nn\\
  &\overset{(a)}\leq\underset{I\in \mathcal I_n^{(a)}}{\max}P_{H_{1,I}}(\mmd^2_{u,I}(Y_{I},Y_{\bar{I}})<t)\nn\\
  &=\underset{I\in \mathcal I_n^{(a)}}{\max}P_{H_{1,I}}(\mmd^2[p,q]-\mmd^2_{u,I}(Y_{I},Y_{\bar{I}})>\mmd^2[p,q]-t)\nn\\
  &\overset{(b)}{\leq} \underset{I\in \mathcal I_n^{(a)}}{\max} \exp{\left(-\frac{(\mmd^2[p,q]-t)^2|I|(n-|I|)}{8K^2n}\right)}
\end{flalign}
where (a) holds by taking $I'=I$, and (b) holds by  applying McDiarmid's inequality.
It can be easily checked that under the condition \eqref{eq:51},
\begin{flalign}
  &\underset{I\in \mathcal I_n^{(a)}}{\max}P(H_0|H_{1,I}) \nn\\
  &\quad\quad\leq \exp{\left(-\frac{(\mmd^2[p,q]-t)^2|I|(n-|I|)}{8K^2n}\right)}\Big\vert_{|I|=n-\frac{16K^2(1+\eta)}{t^2}\underset{\text{arbitrary $k$ number of }\log}{\underbrace{\log\cdots\log\log}n}}\nn\\
&\quad\quad\rightarrow 0, \text{ as } n\rightarrow \infty.
\end{flalign}

Therefore, we conclude that the condition \eqref{eq:51} guarantees that $R_m^{(n)}\rightarrow 0$ as $n\rightarrow \infty$, and thus guarantees the consistency of the test \eqref{eq:test}.
\section{Proof of Theorem \ref{thm:conv}}\label{proof:conv}
The idea  is to consider the following  problem which has lower risk than our original problem, and show that  there exist distributions (in fact for Gaussian $p$ and $q$), under which  such a risk is bounded away from zero for all tests if the necessary conditions are not satisfied.

First consider the following  problem, in which all  candidate anomalous intervals have the same length $k$, and hence there are in total $n-k+1$ candidate anomalous intervals. Furthermore, suppose the distribution $p$ is Gaussian with mean zero and variance one, and the distribution $q$ is Gaussian with mean $\mu>0$ and variance one.
We define the minimax risk of a test for such a problem as follows:
\begin{flalign}
  R_m(k)=P(H_1|H_0)+\underset{|I|=k}{\max}P(H_0|H_{1,I}),
\end{flalign}
and we denote the minimum minimax risk as $R_m^*(k)$.
We further assign uniform distribution over all candidate anomalous intervals  under the alternative hypothesis $H_1$, i.e., each candidate anomalous interval has the same probability $\frac{1}{n-k+1}$ to occur. Thus the Bayes risk is given by
\begin{flalign}
    R_b=P(H_1|H_0)+\frac{1}{n-k+1}\sum_{|I|=k}P(H_0|H_{1,I}) ,
\end{flalign}
and we use $R_b^*$ to denote the minimum Bayes risk over all possible tests.
It is clear that
\[ R_m^*(k)\geq R_b^*.\]

It is justified in \cite{Arias2008} that the optimal  Bayes risk $R_b^*$ can be lower bounded as follows.
\begin{flalign}\label{eq:lala}
  R_b^*\geq 1-\frac{1}{2}\sqrt{\mathbb{E}e^{\mu^2Z}-1},
\end{flalign}
where $Z=|S\bigcap S'|$ with $S$ and $S'$ being  independently and uniformly drawn at random from all candidate anomalous intervals.

We next characterize the distribution of the random variable $Z$  in order to evaluate the lower bound.
We are interested only in the case with $k<\frac{n}{2}$.
It can be shown that
\begin{flalign}
  P(Z=i)&=\frac{2(n-2k+1+i)}{(n-k+1)^2}, \text{ for }1\leq i\leq k-1\nn\\
  P(Z=k)&=\frac{1}{n-k+1}\nn\\
  P(Z=0)&=1-\sum_{i=1}^{k-1}\frac{2(n-2k+1+i)}{(n-k+1)^2}-\frac{1}{n-k+1}.
\end{flalign}

Based on the above distribution of $Z$, we obtain
\begin{flalign}\label{eq:emuz}
  \mathbb{E}e^{\mu^2Z}&-1\nn\\
  =&\sum_{j=1}^{k-1}\frac{2(n-2k+1+j)e^{\mu^2j}}{(n-k+1)^2}+\frac{e^{\mu^2k}}{n-k+1}+1-\frac{2(k-1)(n-\frac{3}{2}k+1)+n-k+1}{(n-k+1)^2}-1\nn\\
  =&\frac{2(n-2k+1)}{(n-k+1)^2}\sum_{j=1}^{k-1}e^{\mu^2j}+\frac{2}{(n-k+1)^2}\sum_{j=1}^{k-1}je^{\mu^2j}+\frac{e^{\mu^2k}}{n-k+1}\nn\\
  &-\frac{2(k-1)(n-\frac{3}{2}k+1)+n-k+1}{(n-k+1)^2}\nn\\
  \overset{(a)}{\leq}&\frac{2(n-2k+1)}{(n-k+1)^2}\int_1^k e^{\mu^2x}dx+\frac{2}{(n-k+1)^2}\int_1^k xe^{\mu^2x}dx +\frac{e^{\mu^2k}}{n-k+1}\nn\\
  &-\frac{2(k-1)(n-\frac{3}{2}k+1)+n-k+1}{(n-k+1)^2}\nn\\
  =&\frac{2(n-2k+1)}{(n-k+1)^2}(e^{\mu^2k}-e^{\mu^2})+\frac{2}{(n-k+1)^2}(\frac{1}{\mu^2}ke^{\mu^2k}-\frac{1}{\mu^4}e^{\mu^2k}-\frac{1}{\mu^2}e^{\mu^2}+\frac{1}{\mu^4}e^{\mu^2})\nn\\
  & +\frac{e^{\mu^2k}}{n-k+1}-\frac{2(k-1)(n-\frac{3}{2}k+1)+n-k+1}{(n-k+1)^2}.
\end{flalign}

It can be checked that
if $k\leq \frac{1}{2\mu^2}\log n$, \eqref{eq:emuz} converges to zero as $n$ goes to infinity. Hence,
$R_b^*\geq 1$ as $n$ goes to infinity, which further implies that $R_m^*(k)>1$, as $n\rightarrow \infty$, and thus any  test is no better than random guess.
Since $\mu$ can be any constant, there always exists Gaussian $p$ and $q$ such that no test can be consistent as long as $k\leq c\log n$, where $c$ is any constant.



Now consider the original problem with the minimax risk
\begin{flalign}
  R_m=P(H_1|H_0)+\underset{I_{\min}\leq|I|\leq I_{\max}}{\max}P(H_0|H_{1,I}).
\end{flalign}
It can be shown that
\[R_m^*\geq R^*(k),\text{ if }k=I_{\min}\]
where $R_m^*$ denotes the minimum risk over all possible tests. Based on the above argument on $R^*(k)$, it is clear that if $I_{\min}\leq c\log n$,  there exists no test such that $R_m^*$ converges to zero as $n$ goes to infinity for  arbitrary distributions $p$ and $q$.

Furthermore, consider the case with only one candidate anomalous interval $I$ with length $k$.
The risk in this case is
\begin{flalign}
  R(k)=P(H_1|H_0)+P(H_0|H_{1,I})
\end{flalign}
where $|I|=k$. It is also clear that $R_m^*\geq R^*(k)$ where $k=I_{\max}$. For such a simple case, the problem reduces to the two-sample problem, detecting whether the samples in the interval $I$ and the samples outside of the interval $I$ are generated from the same distribution. In order to guarantee $R^*(k)\rightarrow 0$ as $n\rightarrow \infty$,
$k$ and $n-k$  should both scale with $n$ to infinity.
Thus, in order to guarantee $R_m^*\rightarrow 0$, as $n\rightarrow\infty$, $n-I_{\max}\rightarrow \infty$ is necessary for any test to be universally consistent.
This concludes the proof.

\section{Proof of Sufficient Conditions for Ring Networks}\label{proof:ring}
Following steps similar to those in Appendix \ref{app:line}, we derive the following bound
\begin{flalign}
  P(H_1|H_0)&\leq \sum_{I\in \mathcal I_n^{(a)}} \exp{\left(-\frac{2t^2|I|(n-|I|)}{16nK^2}\right)}\nn\\
  &\overset{(a)}{=}\sum_{i=I_{\min}}^{I_{\max}}n\exp{\left(-\frac{2t^2i(n-i)}{16nK^2}\right)}\nn\\
  &\overset{(b)}{\leq} \sum_{i=I_{\min}}^{I_{\max}}n\exp{\left(-\frac{2t^2\min\{I_{\min}(n-I_{\min}),I_{\max}(n-I_{\max})\}}{16nK^2}\right)}\nn\\
  &\overset{(c)}{\leq} n^2\exp{\left(-\frac{2t^2\min\{I_{\min}(n-I_{\min}),I_{\max}(n-I_{\max})\}}{16nK^2}\right)}\nn\\
  &=\exp{\left(2\log n-\frac{2t^2\min\{I_{\min}(n-I_{\min}),I_{\max}(n-I_{\max})\}}{16nK^2}\right)}
\end{flalign}
where (a) is due to the fact in the ring network,  there are $n$   candidate anomalous intervals with size $i$, (b) is due to the fact that $i(n-i)$ is lowered bounded by $\min\{I_{\min}(n-I_{\min}),I_{\max}(n-I_{\max})\}$, and (c) is due to the fact that $I_{\max}-I_{\min}\leq n$.

It can be checked that $P(H_1|H_0)\rightarrow 0$ as $n\rightarrow \infty$ if
\begin{flalign}
  \frac{16K^2(1+\eta)}{t^2}\log n\leq I_{\min}\leq I_{\max}\leq n-\frac{16K^2(1+\eta)}{t^2}\log n.
\end{flalign}

Furthermore, following steps similar to those in Appendix \ref{app:line}, we can derive an upper bound on the type II error and show that it converges to zero as $n\rightarrow \infty$ if
\begin{flalign}
  I_{\min}\rightarrow \infty,   \quad n-I_{\max}\rightarrow \infty,
\end{flalign}

Combining the two conditions completes the proof.

\section{Proof of Necessary Conditions for Ring Networks}\label{proof:conv_ring}
The proof  follows the idea in Appendix \ref{proof:conv} for the line network. Here, we consider  a problem
in which all the candidate anomalous intervals have the same length $k$, i.e., there are in total $n-k+1$ candidate anomalous intervals. Furthermore, suppose the distribution $p$ is Gaussian with mean zero and variance one, and the distribution $q$ is Gaussian with mean $\mu>0$ and variance one.
We define the minimax risk of a test for such a problem as follows:
\begin{flalign}
  R_m(k)=P(H_1|H_0)+\underset{|I|=k}{\max}P(H_0|H_{1,I}),
\end{flalign}
and we denote the minimum minimax risk as $R^*(k)$.
We further assign uniform distribution over all candidate anomalous intervals under the alternative hypothesis $H_1$, i.e., each candidate anomalous interval has the same probability $\frac{1}{n-k+1}$ to be anomalous. Hence, the Bayes risk is given by
\begin{flalign}
    R_b=P(H_1|H_0)+\frac{1}{n}\sum_{|I|=k}P(H_0|H_{1,I}) ,
\end{flalign}
and we use $R_b^*$ to denote the minimum Bayes risk over all possible tests.
It is clear that
\[R_k^*\geq R_b^*.\]

In order to apply the same property in \eqref{eq:lala}, we characterize the distribution of the random variable $Z$  as follows.
\begin{flalign}
  P(Z=i)&=\frac{2}{n}, \text{ for }i=1,\ldots,k-1,\nn\\
  P(Z=k)&=\frac{1}{n},\nn\\
  P(Z=0)&=\frac{n-2k+1}{n}.
\end{flalign}
Then we have
\begin{flalign}
  \mathbb{E}e^{\mu^2Z}&-1\nn\\
  =&\sum_{i=1}^{k-1}\frac{2}{n}e^{\mu^2i}+\frac{1}{2}e^{\mu^2k}+\frac{n-2k+1}{n}-1\nn\\
  \leq& \frac{2}{n\mu^2}e^{\mu^2k}-\frac{2}{n}e^{\mu^2}+\frac{1}{n}e^{\mu^2k}-\frac{2k-1}{n}\nn\\
  \rightarrow & 0, \text{ if } k\leq \mathcal O(\log n).
\end{flalign}

By  arguments similar to those for the line network, we conclude that $I_{\min}>c\log n$ for any constant $c$ and $n-I_{\max}\rightarrow \infty$ as $n\rightarrow \infty$ are necessary to guarantee any  test to be universally consistent for arbitrary distributions $p$ and $q$.

\bibliography{kernelmean}

\begin{thebibliography}{10}
\providecommand{\url}[1]{#1}
\csname url@samestyle\endcsname
\providecommand{\newblock}{\relax}
\providecommand{\bibinfo}[2]{#2}
\providecommand{\BIBentrySTDinterwordspacing}{\spaceskip=0pt\relax}
\providecommand{\BIBentryALTinterwordstretchfactor}{4}
\providecommand{\BIBentryALTinterwordspacing}{\spaceskip=\fontdimen2\font plus
\BIBentryALTinterwordstretchfactor\fontdimen3\font minus
  \fontdimen4\font\relax}
\providecommand{\BIBforeignlanguage}[2]{{%
\expandafter\ifx\csname l@#1\endcsname\relax
\typeout{** WARNING: IEEEtran.bst: No hyphenation pattern has been}%
\typeout{** loaded for the language `#1'. Using the pattern for}%
\typeout{** the default language instead.}%
\else
\language=\csname l@#1\endcsname
\fi
#2}}
\providecommand{\BIBdecl}{\relax}
\BIBdecl

\bibitem{zou2014kernel}
S.~Zou, Y.~Liang, and H.~V. Poor, ``A kernel-based nonparametric test for
  anomaly detection over line networks,'' in \emph{IEEE International Workshop
  on Machine Learning for Signal Processing (MLSP)}, 2014.

\bibitem{Berl2004}
A.~Berlinet and C.~Thomas-Agnan, \emph{Reproducing Kernel {H}ilbert Spaces in
  Probability and Statistics}.\hskip 1em plus 0.5em minus 0.4em\relax Kluwer,
  2004.

\bibitem{Srip2010}
B.~Sriperumbudur, A.~Gretton, K.~Fukumizu, G.~Lanckriet, and
  B.~Sch$\ddot{\text{o}}$lkopf, ``Hilbert space embeddings and metrics on
  probability measures,'' \emph{J. Mach. Learn. Res.}, vol.~11, pp. 1517--1561,
  2010.

\bibitem{Scho2002}
B.~Sch$\ddot{\text{o}}$lkopf and A.~Smola, \emph{Learning with Kernels}.\hskip
  1em plus 0.5em minus 0.4em\relax MA: MIT Press, 2002.

\bibitem{Fuku2008}
K.~Fukumizu, A.~Gretton, X.~Sun, and B.~Sch$\ddot{\text{o}}$lkopf, ``Kernel
  measures of conditional dependence,'' in \emph{Proc. Advances in Neural
  Information Processing Systems (NIPS)}, 2008.

\bibitem{Gretton2012}
A.~Gretton, K.~Borgwardt, M.~Rasch, B.~Sch$\ddot{\text{o}}$lkopf, and A.~Smola,
  ``A kernel two-sample test,'' \emph{J. Mach. Learn. Res.}, vol.~13, pp.
  723--773, 2012.

\bibitem{Arias2005}
E.~Arias-Castro, D.~L. Donoho, and X.~Huo, ``Near-optimal detection of
  geometric objects by fast multiscale methods,'' \emph{IEEE Trans. Inform.
  Theory}, vol.~51, no.~7, pp. 2402--2425, Jul. 2005.

\bibitem{Walt2010}
G.~Walther, ``Optimal and fast detection of spatial clusters with scan
  statistics,'' \emph{Ann. Statist.}, vol.~38, no.~2, pp. 1010--1033, 2010.

\bibitem{Paci2004}
P.~M. Pacifico, C.~Genovese, I.~Verdinelli, and L.~Wasserman, ``False discovery
  control for random fields,'' \emph{J. Amer. Stat. Assoc.}, vol.~99, pp.
  1002--1014, 2004.

\bibitem{Arias2008}
E.~Arias-Castro, E.~J. Candes, H.~Helgason, and O.~Zeitouni, ``Searching for a
  trail of evidence in a maze,'' \emph{Ann. Statist.}, vol.~36, no.~4, pp.
  1726--1757, 2008.

\bibitem{Adda2010}
L.~Addario-Berry, N.~Broutin, L.~Devroye, and G.~Lugosi, ``On combinatorial
  testing problems,'' \emph{Ann. Statist.}, vol.~38, no.~5, pp. 3063--3092,
  2010.

\bibitem{Arias2011}
E.~Arias-Castro, E.~J. Candes, and A.~Durand, ``Detection of an anomalous
  cluster in a network,'' \emph{Ann. Statist.}, vol.~39, no.~1, pp. 278--304,
  2011.

\bibitem{Sharp2013a}
J.~Sharpnack, A.~Rinaldo, and A.~Singh, ``Changepoint detection over graphs
  with the spectral scan statistic,'' in \emph{Proc. International Conference
  on Artifical Intelligence and Statistics (AISTATS)}, Scottsdale, AZ, May
  2013.

\bibitem{Sharp2013b}
------, ``Detecting activations over graphs using spanning tree wavelet
  bases,'' in \emph{Artificial Intelligence and Statistics (AISTATS)},
  Scottsdale, AZ, May 2013.

\bibitem{qian2014}
J.~Qian, V.~Saligrama, and Y.~Chen, ``Connected sub-graph detection,'' in
  \emph{Proc. International Conference on Artifical Intelligence and Statistics
  (AISTATS)}, 2014, pp. 796--804.

\bibitem{qian2014efficient}
J.~Qian and V.~Saligrama, ``Efficient minimax signal detection on graphs,'' in
  \emph{Proc. Advances in Neural Information Processing Systems (NIPS)}, 2014,
  pp. 2708--2716.

\bibitem{Srip2008}
B.~Sriperumbudur, A.~Gretton, K.~Fukumizu, G.~Lanckriet, and
  B.~Sch$\ddot{\text{o}}$lkopf, ``Injective {H}ilbert space embeddings of
  probability measures,'' in \emph{Proc. Annual Conference on Learning Theory
  (COLT)}, 2008.

\bibitem{Fuku2009}
K.~Fukumizu, B.~Sriperumbudur, A.~Gretton, and B.~Sch$\ddot{\text{o}}$lkopf,
  ``Characteristic kernels on groups and semigroups,'' in \emph{Proc. Advances
  in Neural Information Processing Systems (NIPS)}, 2009.

\bibitem{Song2013}
L.~Song, A.~Gretton, and K.~Fukumizu, ``Kernel embeddings of conditional
  distributions,'' \emph{IEEE Signal Processing Magazine}, vol.~30, no.~4, pp.
  98--111, 2013.

\bibitem{Song2011a}
L.~Song, A.~Parikh, and E.~P. Xing, ``Kernel embeddings of latent tree
  graphical models,'' in \emph{Proc. Advances in Neural Information Processing
  Systems (NIPS)}, 2011.

\bibitem{Song2011b}
L.~Song, A.~Gretton, D.~Bickson, Y.~Low, and C.~Guestrin, ``Kernel belief
  propagation,'' in \emph{Proc. International Conference on Artifical
  Intelligence and Statistics (AISTATS)}, 2011.

\bibitem{Smola2007}
A.~J. Smola, A.~Gretton, L.~Song, and B.~Sch$\ddot{\text{o}}$lkopf, ``A
  {H}ilbert space embedding for distributions,'' \emph{{\em Algorithmic
  Learning Theory, E. Takimoto (Eds.),} Lecture Notes on Computer Science, {\em
  Springer}}, 2007.

\bibitem{Fried1979}
J.~H. Friedman and L.~C. Rafsky, ``Multivariate generalizations of the
  wald-wolfowitz and smirnov two-sample tests,'' \emph{Ann. Statist.}, vol.~7,
  no.~4, pp. pp. 697--717, 1979.

\bibitem{Hall2002}
\BIBentryALTinterwordspacing
P.~Hall and N.~Tajvidi, ``\BIBforeignlanguage{English}{Permutation tests for
  equality of distributions in high-dimensional settings},''
  \emph{\BIBforeignlanguage{English}{Biometrika}}, vol.~89, no.~2, pp. pp.
  359--374, 2002. [Online]. Available:
  \url{http://www.jstor.org/stable/4140582}
\BIBentrySTDinterwordspacing

\bibitem{Harchaoui2008}
Z.~Harchaoui, F.~Bach, and E.~Moulines, ``Testing for homogeneity with kernel
  fisher discriminant analysis,'' in \emph{Proc. Advances in Neural Information
  Processing Systems (NIPS)}, 2008.

\bibitem{Sugiyama2012}
M.~Sugiyama, T.~Suzuki, and T.~Kanamori, \emph{Density Ratio Estimation in
  Machine Learning}.\hskip 1em plus 0.5em minus 0.4em\relax New York, NY, USA:
  Cambridge University Press, 2012.

\end{thebibliography}
\bibliographystyle{IEEEtran}

\renewcommand{\baselinestretch}{1}

\end{document}